\documentclass[runningheads]{llncs}

\usepackage{amsmath,amssymb,amsfonts}
\usepackage{bbold}
\usepackage{graphicx}
\usepackage[]{algorithm}
\usepackage[]{algorithmic}

\DeclareMathOperator*{\argmin}{arg\,min}
\newcommand{\R}[1]{\mathbb{R}^{#1}}

\newcommand{\Es}[2]{\mathbb{E}_{#1}\left[#2\right]}
\newcommand{\JJ}{\mathcal{J}}
\newcommand{\system}[2]{\left \lbrace \begin{aligned}
& #1 \\ & #2
\end{aligned} \right.}
\newcommand{\pare}[1]{\left({#1}\right)}
\newcommand{\abs}[1]{\left \lvert{#1}\right \rvert}

\usepackage{tikz}
\usetikzlibrary{shapes.misc,decorations.pathreplacing}
\tikzset{cross/.style={cross out, draw=black, fill=none, minimum size=2*(#1-\pgflinewidth), inner sep=0pt, outer sep=0pt}, cross/.default={2pt}}

\begin{document}

\title{Reinforcement Learning for Variable Selection in a Branch and Bound Algorithm}

\titlerunning{Reinforcement Learning for Variable Selection}

\author{Marc Etheve\inst{1,3}\orcidID{0000-0003-4436-3391} \and
Zacharie Al\`es\inst{2,3}\orcidID{0000-0003-4602-2638} \and
C\^ome Bissuel\inst{1}\orcidID{0000-0002-5430-3168} \and 
Olivier Juan\inst{1}\orcidID{0000-0003-3445-4847} \and 
Safia Kedad-Sidhoum\inst{3}\orcidID{0000-0002-2184-2261}}

\authorrunning{M. Etheve et al.}

\institute{EDF R\&D, France \\
\email{\{marc.etheve, come.bissuel, olivier.juan\}@edf.fr}
    \and ENSTA Paris, Institut Polytechnique de Paris, France \\
    \email{zacharie.ales@ensta-paris.fr}
        \and CNAM Paris, CEDRIC, France\\
        \email{safia.kedad\_sidhoum@cnam.fr}}

\maketitle

\begin{abstract}
Mixed integer linear programs are commonly solved by \linebreak Branch and Bound algorithms. A key factor of the efficiency of the most successful commercial solvers is their fine-tuned heuristics. In this paper, we leverage patterns in real-world instances to learn from scratch a new branching strategy optimised for a given problem and compare it with a commercial solver.
We propose FMSTS, a novel Reinforcement Learning approach specifically designed for this task. The strength of our method lies in the consistency between a local value function and a global metric of interest. In addition, we provide insights for adapting known RL techniques to the Branch and Bound setting, and present a new neural network architecture inspired from the literature. To our knowledge, it is the first time Reinforcement Learning has been used to fully optimise the branching strategy. Computational experiments show that our method is appropriate and able to generalise well to new instances.
\end{abstract}

\section{Introduction}
Mixed Integer Linear Programming (MILP) is an active field of research due to its tremendous usefulness in real-world applications. The most common method designed to solve MILP problems is the Branch and Bound (B\&B) algorithm (see~\cite{wolsey} for an exhaustive introduction). B\&B is a general purpose procedure dedicated to solve any MILP instance, based on a divide and conquer strategy and driven by generic heuristics and bounding procedures. \\
Recently, a lot of attention has been paid to the interactions between MILP and machine learning. As pointed out in~\cite{ROAI:tour_horizon}, learning methods may compensate for the lack of mathematical understanding of the B\&B method and its variants~\cite{BranchAndPrice,BranchAndCut}. The plethora of different approaches in this young field of research gives evidence of the variety of ways in which learning can be leveraged. For instance, a natural idea is to bypass the whole B\&B procedure to directly learn solutions of MILP instances~\cite{learn_solution}. If one wants to preserve the optimality guarantee provided by the B\&B algorithm, a solution could be to rather learn the output of a computationally expensive heuristic used in a B\&B scheme~\cite{khalil_learningBranch,balcan_learningBranch,gasse_learningBranch}. Alternatively,~\cite{learningCut} suggests learning to select the best cut among a set of available cuts at each node of the B\&B tree. Whether it is by Imitation Learning or by Reinforcement Learning (RL), these solutions are often limited by their scope: they seek to take decisions according to a local criterion. \\

In the present work, we propose FMSTS (\textit{Fitting for Minimising the SubTree Size}), a novel approach based on Reinforcement Learning aiming at optimising a global criterion at the scale of the whole B\&B tree. We learn a branching strategy from scratch, independent of any heuristic.\\
The paper is structured as follows. First, we define the general setting of our study. Using RL to minimise a global criterion, we then demonstrate that, under certain assumptions, a specific kind of value functions enforces the optimality of such criterion. Next, we propose to adapt known generic learning methods and neural network architectures to the Branch and Bound setting. We illustrate our proposed method on industrial problems and discuss it before concluding.

\section{General Setting}

In real-world applications, companies often optimise systems on a regular basis given fluctuating data. This case has been studied in the literature for different purposes, such as learning an approximate solution~\cite{learn_solution} or imitating heuristics~\cite{gasse_learningBranch}. However, to our knowledge, no concrete contribution has been made regarding the use of Reinforcement Learning for variable selection (branching) in this setting. The present work fills this gap.\\

Throughout this paper, we are interested in the following setting. For a given problem $\mathcal{P}$, the instances are perceived as randomly distributed according to an unknown distribution $\mathcal{D}$. This distribution, emanating from real-world systems, governs the fluctuating data $(A,b,c)$ across instances, written as
\begin{equation}
    p \in \mathcal{P} : \system{\min_{x \in \R{n}} c^\top x}{s.t. \ 
    \begin{aligned}
        \ & Ax \leq b \\
        \ & x_\JJ \in \{0,1\}^{|\JJ|}, \ x_{-\JJ} \in \R{n-|\JJ|}
    \end{aligned}}
    \label{milp}
\end{equation}
with $A \in \R{m \times n}$, $b \in \R{m}$ and $c \in \R{n}$. In practice, as the instances come from a single problem, they share the same structure. Especially, the set of null coefficients, the number of constraints $m$, of variables $n$ and the set of binary variables $\JJ$ are the same for every instance of a given problem.\\

In this setting, we seek to learn and optimise a branching strategy to solve to optimality any instance of a given problem. As pointed out in~\cite{art:node_select_ML}, the problem of optimising the decisions along a B\&B tree is naturally formulated as a control problem on a sequential decision-making process. More specifically, it is equivalent to solving a finite-horizon deterministic Markov Decision Process (MDP) and may thus be tackled by Reinforcement Learning (see~\cite{SuttonBarto} for an introduction).

\section{Fitting for Minimising the SubTree Size (FMSTS)}

In a finite-horizon setting, Reinforcement Learning aims at optimising an agent to produce sequences of actions that achieve a global objective. The agent is guided by local costs associated with the actions it takes. Starting from an initial state in a possibly stochastic environment, it must learn to take the best sequence of actions and thus transitioning from state to state to minimise the overall costs. Exact Q-learning solves such problem by updating a table mapping (Q-function) from state/action pairs to discounted future costs (Q-values). \\
In the following, we define the RL problem of interest and, as exact Q-learning is not practicable, the approximate framework considered. Next, we propose a specific informative Q-function allowing us to use this framework in practice.

\subsection{Approximate Q-learning with Observable Q-values}

Let us denote by $\mathcal{S}$ the set of every reachable state for a given problem $\mathcal{P}$, a state being defined as all the information available when taking a branching decision in a B\&B tree. Under perfect information, a state associated to a B\&B node is the whole B\&B tree as it has been expanded at the time the branching decision is taken. We write $\mathcal{A}$ the set of actions, \textit{i.e.} the set of available branching decisions on a specific problem (the set of binary variables $\mathcal{A}=\JJ$ in our case) and $\pi$ a policy mapping any state to a branching decision:
$$
\pi \ : \ \system{\mathcal{S} \rightarrow \mathcal{A}}
                {s \mapsto \pi(s) = a}.
$$
The transitions between states are governed by the B\&B solver, and the MDP is regarded as deterministic: given an instance and a state, performing an action will always lead to the same next state. In practice, such assumption is met as soon as the solver's decisions (apart from branching) are non stochastic.\\
We call agent any generator $\Pi$ of branching sequences following policy $\pi$ and denote $\Pi(p)$ the sequence of decisions that maps an instance $p$ to a complete B\&B tree.
Let $\mu \pare{\Pi(p)}$ be any metric of interest on the tree generated by $\Pi$ for an instance $p$, and assume this metric is to be minimised. For instance, we can think of $\mu$ as the size of the generated tree, the number of simplex iterations, etc. In this setting, we are looking for the $\mu-$optimal agent $\Pi^*$ such that 
\begin{equation}
\Pi^* \in \argmin_{\Pi} \Es{p \sim \mathcal{D}}{\mu \pare{ \Pi \pare{p}}}.
\label{eq:generic_optimal_parameter}
\end{equation}
Note here that the expectation is only on the MILP instances, as the MDP is deterministic. \\

Let us assume that one can define a Q-function $Q^\pi$ which is consistent with $\mu$, in the sense that $\mu$ is minimised if $\pi(s) = \argmin_{a \in \mathcal{A}}Q^\pi\pare{s,a}$. Even in this case, exact Q-learning cannot be used to minimise $\mu$. First, maintaining an exact table for the Q-function is not tractable due to the size of $\mathcal{S}$, including for small real-world problems. Second, the transition from a state to the next is too complex to be modeled since it partly results from a linear optimisation.\\
To bypass these problems, we approximate the Q-function (see~\cite{SuttonBarto} for an introduction to Approximate Q-learning) by a neural network $\hat{Q}$ parametrised by $\theta$ and optimised by a dedicated gradient method as in the DQN (Deep Q-Network) approach~\cite{DQN}.
We define the policy $\pi_\theta$ resulting from the Q-network as $\pi_\theta(s) = \argmin_{a \in \mathcal{A}}\hat{Q}(s,a;\theta)$. \\
When facing a deterministic MDP, the exact Q-value $Q^{\pi_\theta}(s,a)$ of an action $a$ given a state $s$ and a policy $\pi_\theta$ is not stochastic and thus may be observable. In that case, the classic Temporal Difference loss used in~\cite{DQN} for training the Q-Network comes down to the simple expression
\begin{equation}
    L(\theta) = \Es{s,a\sim \rho(.)}{\pare{Q^{\pi_{\theta}}(s,a)-\hat{Q}(s,a;\theta)}^2}
    \label{eq:lossMFSTS}
\end{equation}
where $\rho$ is the behaviour distribution of our agent, as stated in~\cite{DQN}. Note that, in Equation \eqref{eq:lossMFSTS}, the observed Q-values are naturally influenced by parameter $\theta$ through the policy. Such loss is actually intuitive: if $Q^{\pi_\theta}$ is consistent with $\mu$, if each action has non-zero probability to be taken in any encountered state and if $L(\theta)=0$, then each B\&B tree built by agent $\Pi_\theta$ (following $\pi_\theta$) is optimal with respect to $\mu$ with probability 1.

\subsection{Using the Subtree Size as Value Function}

As highlighted in~\cite{gasse_learningBranch}, an important difficulty when applying Reinforcement Learning to B\&B algorithms is the credit assignment problem~\cite{credit_assignment_minsky}: in order to determine the actions that lead to a specific outcome, one may define non-sparse informative local costs (negative rewards) consistent with the global objective. This is not mandatory in a RL setting but may facilitate the learning task. \\

We choose the number of nodes in the generated tree as the global metric $\mu$. This metric is often used to compare B\&B methods (see for instance~\cite{khalil_learningBranch,balcan_learningBranch,gasse_learningBranch}), as it is a proxy for computational efficiency and independent of hardware considerations. \\
One of the main contributions of this paper is to propose a local Q-function $Q^\pi$ which is consistent with the chosen global metric $\mu$. We take advantage of the deterministic aspect of the environment and define $Q^\pi(s,a)$ as the size of the subtree rooted in the B\&B node corresponding to $s$ generated by action $a$ and policy $\pi$. As stated in Proposition~\ref{prop:suboptimalityMSTS}, this particular Q-function is not consistent in general with our choice of global criterion $\mu$. Nonetheless, Proposition~\ref{prop:optimalityMSTS} asserts its optimality when using Depth First Search as node selection strategy.
\begin{proposition}
In general, minimising the size of the subtree under any node in a B\&B tree is not optimal with respect to the tree size.
\label{prop:suboptimalityMSTS}
\end{proposition}
\noindent The proof is omitted for the sake of conciseness, but one can prove that minimising the subtree size can be sub-optimal when using Breadth First Search as the node selection strategy.
\begin{proposition}
When using Depth First Search (DFS) as the node selection strategy, minimising the whole B\&B tree size is achieved when any subtree is of minimal size.
\label{prop:optimalityMSTS}
\end{proposition}
\begin{proof} 
Let us call $\mathcal{O}$ the set of open nodes at a given iteration of the B\&B process for a specific instance. The set of closed nodes (either by pruning or branching) is denoted $\mathcal{C}$. \\
We write $V^\pi\pare{s \left\lvert \zeta, \eta \right.}$ the size of the subtree under $s$, entirely determined by the policy $\pi$ followed in this subtree, a set of primal bounds $\zeta$ found in other subtrees and a node selection strategy $\eta$. When using DFS, the subtrees under each open node are expanded and fully solved sequentially, thus we can assume with no loss of generality that $\mathcal{O}$ is equal to $\{1,...,k \}$ and is sorted according to the planned visiting order. In that case, the size $V$ of the whole B\&B tree can be expressed as
$$
V = \abs{\mathcal{C}} + \sum_{i=1}^k V^{\pi_i}\pare{s_i \ \left\lvert \{z_0\} \cup \pare{ \bigcup_{j<i} \zeta_j}, \  \eta=DFS\right.}
$$
with $\zeta_i$ the set of all bounds to be found in the subtree rooted in $s_i$ and $z_0$ the best bound obtained in $\mathcal{C}$.\\
It remains to prove that $\pi_1$ is optimal only if it leads to the minimal subtree under $s_1$.
As two separate subtrees can only affect each other through their best primal bound under DFS, we have
$$
V= \abs{\mathcal{C}} + V^{\pi_1}\pare{s_1 \ \left\lvert \{z_0\}, \  \eta=DFS\right.} + \sum_{i=2}^k V^{\pi_i}\pare{s_i \ \left\lvert \{z_{i-1}\}, \  \eta=DFS\right.}
$$
with $z_{i-1} = \min \left\lbrace \{z_0\} \cup \pare{ \bigcup_{j<i} \zeta_j} \right\rbrace$.\\
Since the B\&B procedure (with a gap set to zero) guarantees that we find the best primal bound of any expanded subtree, $\pare{z_i}_{i=1}^k$ are completely independent of the branching policies, which gives, for any $\pi_j, \ j \in \{2,...,k\}$:
$$
\argmin_{\pi_1 \in \Pi_1} V = \argmin_{\pi_1 \in \Pi_1} V^{\pi_1}\pare{s_1 \ \left\lvert \{z_0\}, \  \eta=DFS\right.}
$$ 
with $\Pi_1$ the set of all valid branching policies under $s_1$. Therefore, choosing any other policy than $\pi_1 \in \argmin_{\pi_1 \in \Pi_1} V^{\pi}\pare{s_1 \ \left\lvert \{z_0\}, \  \eta=DFS\right.}$ is sub-optimal with respect to the tree size.\\ \mbox{}\hfill$ \qed $
\end{proof}
In the remaining, we use DFS as the node selection strategy according to Proposition~\ref{prop:optimalityMSTS}. We now focus on optimising the branching strategy (variable selection) to minimise at each node the size of the underlying subtree. If we write $D_0^{\pi(s)}(s)$ and $D_1^{\pi(s)}(s)$ the child nodes of $s$ following policy $\pi$, such value function satisfies the Bellman Equation \eqref{eq:bellman}. The relationship between the value and the Q-function is trivially defined by $Q^\pi(s,a) = 1 + V^\pi(D_0^a(s)) + V^\pi(D_1^a(s))$.\\
\begin{equation}
    V^\pi(s) = 1 + V^\pi \pare{D_0^{\pi(s)}(s)}  + V^\pi \pare{D_1^{\pi(s)}(s)}
    \label{eq:bellman}
\end{equation}
This value function has two advantages. First, it is observable as assumed earlier: we only need to count the number of inheriting nodes once the B\&B tree is fully expanded. Second, it is a local objective which guarantees the optimality of a global criterion, hence allowing us to perform RL without designing a sub-optimal reward using any domain knowledge.

\subsection{Algorithm}

Using Approximate Q-learning and the subtree size as value function leads
us to propose the FMSTS algorithm (Algorithm \ref{algo:FMSTS}). Using Experience Replay and $\varepsilon$-greedy exploration as in~\cite{DQN}, the algorithm essentially boils down to consecutively solving a MILP instance following the current policy or random choices with probability $\varepsilon$, fitting the observed values sampled from an experience replay buffer and iterating with the updated policy.

\begin{algorithm}[H]
\caption{FMSTS}
\begin{algorithmic}
 \FOR{t = 0,...,N-1}
  \STATE Draw randomly an instance $p$.
  \STATE Solve $p$ following $\pi_{\theta_{t}}$ with $\varepsilon$-greedy exploration.
  \STATE Collect experiences along the generated tree $\pare{s_i, a_i, Q^{\pi_{\theta_t}}(s,a),\hat{Q}(s,a;\theta_t)}$ and store them into an experience replay buffer $\mathcal{B}$.
  \STATE Update to $\theta_{t+1}$ using loss \eqref{eq:lossMFSTS} on experiences drawn from $\mathcal{B}$.
  \ENDFOR
 \end{algorithmic}
 \label{algo:FMSTS}
\end{algorithm}

\section{Adapting Learning to the Branch and Bound Setting}

To ensure the success of the FMSTS method (Algorithm~\ref{algo:FMSTS}) with respect to the objective \eqref{eq:generic_optimal_parameter}, we need to adapt some components to the Branch and Bound setting. First, we adapt the loss guiding the neural network's training. Next, we use Prioritized Experience Replay while normalising probabilities. Last, we propose a new neural network architecture inspired from the literature.

\subsection{Minimising an Expectation on the Instance Distribution}

The loss defined by Equation~\eqref{eq:lossMFSTS} does not seem to correspond to our objective~\eqref{eq:generic_optimal_parameter}. Indeed, it naturally gives more importance to the biggest trees, which can be heavily instance dependent. To neutralise this effect, we weight the loss by the inverse of the size of the corresponding B\&B tree generated by the agent:
\begin{equation}
    L(\theta) = \Es{s,a\sim \rho(.)}{\frac{1}{V^{\pi_\theta}\pare{r(s)}}\pare{Q^{\pi_\theta}(s,a)-\hat{Q}(s,a;\theta)}^2}
    \label{eq:wlossMFSTS}
\end{equation}
with $r(s)$ the root node of the tree containing $s$, such that $V^{\pi_\theta}\pare{r(s)}$ corresponds to the size of this tree. Then, any instance has equal weight in loss~\eqref{eq:wlossMFSTS}.

\subsection{Performing Prioritized Experience Replay}

Prioritized Experience Replay~\cite{PER} biases the uniform replay sampling of Experience Replay~\cite{DQN} towards experiences with high Temporal Difference errors, \textit{i.e.} when the predicted Q-values are far from their target. In FMSTS, an experience is a 4-tuple $\pare{s_j,a_j,Q^{\pi_{\theta_j}}(s_j,a_j), \hat{Q}(s_j,a_j;\theta_j)}$ and the target $Q^{\pi_{\theta_j}}(s_j,a_j)$ is observed, which reduces the error to the simple expression $\abs{Q^{\pi_{\theta_j}}(s_j,a_j) - \hat{Q}(s_j,a_j;\theta_j)}$.\\
In the context of sampling experiences in a B\&B tree, one should take into account that the scale of the target $Q^{\pi_{\theta_j}}$ may vary exponentially both along the tree and across instances. As the scale of the error may likely vary with that of the target, we normalise this error by the target to get the sampling probability in the experience replay buffer
\begin{equation}
    p_j \propto \frac{\abs{Q^{\pi_{\theta_j}}(s_j,a_j) - \hat{Q}(s_j,a_j;\theta_j)}}{Q^{\pi_{\theta_j}}(s_j,a_j)}.
\end{equation}

\subsection{Designing a Regressor for the Q-function}

As in~\cite{khalil_learningBranch}, we use both static and dynamic features to represent a state. Although many features may be relevant for the states' encoding, we opted to keep them limited in the present work. For static information, we perform a dimension reduction by PCA~\cite{PCA}: each instance is represented as the concatenation of its data $(A,b,c)$ and PCA is applied on the resulting vectors. Our representation also includes the following dynamic features: the node's depth, the distance of the current primal solution to the bounds and the branching state. Concretely, the branching state $B$ is one-hot encoded in a $3\abs{\JJ}$ vector. Let us call $\mathcal{B}_0$ and $\mathcal{B}_1$ the set of variables that have been respectively set to 0 and 1 in the ascendant nodes of the current state.
With no loss of generality, let us assume that $\JJ = \{1,...,J\}$. Then we have $B_j = \mathbb{1}_{x_j \in \mathcal{B}_0}$, $B_{j+J} = \mathbb{1}_{x_j \in \mathcal{B}_1}$ and $B_{j+2J} = 1-B_j-B_{j+J}$ for any $j \in \JJ$.\\
The chosen Q-function is essentially multiplicative, in the sense that the ratio between the targets in two consecutive states may be of magnitude 2 due to the binary tree structure. In addition, the scale of these targets may strongly vary between instances. A basic feedforward neural network, based on summations, may struggle to handle such phenomena. To compensate for these effects and adapt to the B\&B setting, we take inspiration from the Dueling architecture of~\cite{dueling} and propose the Multiplicative Dueling Architecture (MDA). As shown in Figure~\ref{fig:MDA}, MDA implements the product between the $1$-D output of a block of fully-connected layers fed with static features and the $\abs{\JJ}$-D output of a block fed with both static and dynamic features. A linear activation on the $1$-D output allows our agent to capture the variability of the chosen Q-function. 
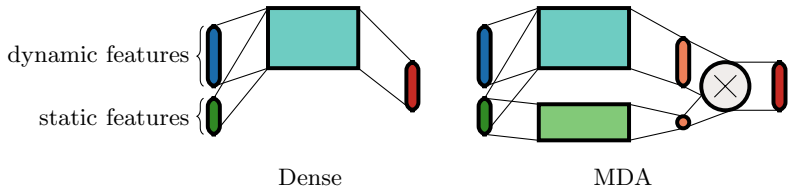
\begin{figure}
\centering
\begin{tikzpicture}[scale=0.8]
    \newcommand{\distLayer}{.8}
    \newcommand{\widthLayer}{.2}
    \newcommand{\widthBlock}{1.5}
    \newcommand{\heightBlock}{1}
    \newcommand{\heightQ}{.8}
    \newcommand{\heightBlockResidual}{.6}
    \newcommand{\rayonSign}{\heightQ/2}
    \newcommand{\linewidthLayer}{1.5}

    \definecolor{colorNonTrain}{rgb}{.94, .93, .92}
    \definecolor{colorPreOutput}{rgb}{.93, .5, .35}
    \definecolor{colorOutput}{rgb}{.78, .17, .14}
    \definecolor{colorInputStatic}{rgb}{.19, .54, .14}
    \definecolor{colorInputDynamic}{rgb}{.1, .42, .67}
    \definecolor{colorBlockStatic}{rgb}{.51, .79, .47}
    \definecolor{colorBlockGlobal}{rgb}{.38, .61, .78}
    \definecolor{colorBlockGlobal}{rgb}{.43, .8, .76}
    
    \draw (\widthLayer/2,0) -- (\widthLayer+\distLayer,.3) ;
    \draw (\widthLayer/2,1) -- (\widthLayer+\distLayer,.3+\heightBlock) ;
    \draw (\widthLayer/2,-.2) -- (\widthLayer+\distLayer,.3+\heightBlock) ;
    \draw (\widthLayer/2,-.8) -- (\widthLayer+\distLayer,.3) ;
    \draw [fill=colorBlockGlobal, draw=black, line width=\linewidthLayer]
        (\widthLayer+\distLayer,.3) --
        ++(\widthBlock,0) --
        ++(0,\heightBlock) --
        ++(-\widthBlock,0) -- 
        ++(0,-\heightBlock) --
        cycle
        {};
    \fill [colorInputDynamic,rounded corners=3.1, draw=black, line width=\linewidthLayer]
        (0,0) --
        ++(\widthLayer,0) --
        ++(0,1) --
        ++(-\widthLayer,0) --
        cycle
        {};
    
    \fill [colorInputStatic,rounded corners=3.1, draw=black, line width=\linewidthLayer]
        (0,-.2) --
        ++(0,-\heightBlockResidual) --
        ++(\widthLayer,0) --
        ++(0,\heightBlockResidual) --
        cycle
        {};
    \draw[decorate, decoration={brace},xshift=-2] (0,0) -- (0,1) node [black,midway, left,xshift=-2] {dynamic features};
    \draw[decorate, decoration={brace,mirror},xshift=-2] (0,-.2) -- (0,-.2-\heightBlockResidual) node [black,midway, left,xshift=-2] {static features};

    \node[align=left] at (0.5*2.5+0.5*\distLayer + 0.5*\widthLayer/2,-.2-\heightBlockResidual -.7) {Dense};
    
    \draw (2.5,.3) -- (2.5+\distLayer + \widthLayer/2,-\heightQ/2) ;
    \draw (2.5,1.3) -- (2.5+\distLayer + \widthLayer/2,\heightQ/2) ;
    
    \fill [colorOutput,rounded corners=3.1, draw=black, line width=\linewidthLayer]
        (2.5+\distLayer,-\heightQ/2) --
        ++(\widthLayer,0) --
        ++(0,\heightQ) --
        ++(-\widthLayer,0) --
        cycle
        {};

    \newcommand{\distFig}{1}
    \newcommand{\deltaFig}{\widthLayer + \distLayer + \widthBlock + \distLayer + \widthLayer + \distFig}
    
    \draw (\deltaFig+\widthLayer/2,0) -- (\deltaFig+\widthLayer+\distLayer,.3) ;
    \draw (\deltaFig+\widthLayer/2,1) -- (\deltaFig+\widthLayer+\distLayer,.3+\heightBlock) ;
    \draw (\deltaFig+\widthLayer/2,-.2) -- (\deltaFig+\widthLayer+\distLayer,.3+\heightBlock) ;
    \draw (\deltaFig+\widthLayer/2,-.8) -- (\deltaFig+\widthLayer+\distLayer,.3) ;
    \draw [fill=colorBlockGlobal, draw=black, line width=\linewidthLayer]
        (\deltaFig+\widthLayer+\distLayer,.3) --
        ++(\widthBlock,0) --
        ++(0,\heightBlock) --
        ++(-\widthBlock,0) -- 
        ++(0,-\heightBlock) --
        cycle
        {};
    
    \draw (\deltaFig+\widthLayer/2,-.2) -- (\deltaFig+\widthLayer+\distLayer,-.3) ;
    \draw (\deltaFig+\widthLayer/2,-.2-\heightBlockResidual) -- (\deltaFig+\widthLayer+\distLayer,-.3-\heightBlockResidual) ;
    
    \draw [fill=colorBlockStatic, draw=black, line width=\linewidthLayer]
        (\deltaFig+\widthLayer+\distLayer,-.3) --
        ++(\widthBlock,0) --
        ++(0,-\heightBlockResidual) --
        ++(-\widthBlock,0) -- 
        ++(0,\heightBlockResidual) --
        cycle
        {};
    \fill [colorInputDynamic,rounded corners=3.1, draw=black, line width=\linewidthLayer]
        (\deltaFig,0) --
        ++(\widthLayer,0) --
        ++(0,1) --
        ++(-\widthLayer,0) --
        cycle
        {};
    
    \fill [colorInputStatic,rounded corners=3.1, draw=black, line width=\linewidthLayer]
        (\deltaFig,-.2) --
        ++(0,-\heightBlockResidual) --
        ++(\widthLayer,0) --
        ++(0,\heightBlockResidual) --
        cycle
        {};
    
    \draw (\deltaFig+2.5,.3) -- (\deltaFig+2.5+\distLayer + \widthLayer/2,0) ;
    \draw (\deltaFig+2.5,1.3) -- (\deltaFig+2.5+\distLayer + \widthLayer/2,\heightQ) ;
    
    \fill [colorPreOutput,rounded corners=3.1, draw=black, line width=\linewidthLayer]
        (\deltaFig+2.5+\distLayer,0) --
        ++(\widthLayer,0) --
        ++(0,\heightQ) --
        ++(-\widthLayer,0) --
        cycle
        {};
     
    \draw (\deltaFig+2.5,-.3) -- (\deltaFig+2.5+\distLayer + \widthLayer/2,-.3-\heightBlockResidual/2 + \widthLayer/2) ;
    \draw (\deltaFig+2.5,-.3-\heightBlockResidual) -- (\deltaFig+2.5+\distLayer + \widthLayer/2,-.3-\heightBlockResidual/2 - \widthLayer/2) ;
    
    \fill [colorPreOutput, draw=black, line width=\linewidthLayer]
        (\deltaFig+2.5+\distLayer+\widthLayer/2,-.3-\heightBlockResidual/2) circle (\widthLayer/2);
    
    \draw (\deltaFig+2.5+ \distLayer + \widthLayer/2,0) -- (\deltaFig+2.5+\distLayer + \widthLayer/2+ \distLayer,-\rayonSign) ;
    \draw (\deltaFig+2.5+ \distLayer + \widthLayer/2,\heightQ) -- (\deltaFig+2.5+\distLayer + \widthLayer/2+ \distLayer,\rayonSign) ;
    
    \draw (\deltaFig+2.5+ \distLayer + \widthLayer/2,-.3-\heightBlockResidual/2 - \widthLayer/2) -- (\deltaFig+2.5+\distLayer + \widthLayer/2+ \distLayer,-\rayonSign) ;
    \draw (\deltaFig+2.5+ \distLayer + \widthLayer/2,-.3-\heightBlockResidual/2 + \widthLayer/2) -- (\deltaFig+2.5+\distLayer + \widthLayer/2+ \distLayer,\rayonSign) ;
    
    \fill [colorNonTrain, draw=black, line width=\linewidthLayer]
        (\deltaFig+2.5+2*\distLayer,0) circle (\rayonSign);
    
    \draw (\deltaFig+2.5+2*\distLayer,0) node[cross=\rayonSign*.4cm,rotate=0,black]{};
    
    \node[align=left] at (0.5*\deltaFig+0.5*2.5+0.5*3*\distLayer +0.5*\widthLayer/2,-.2-\heightBlockResidual -.7) {MDA};    
    \draw (\deltaFig+2.5+2*\distLayer + \widthLayer/2,-\rayonSign) -- (\deltaFig+2.5+3*\distLayer +\widthLayer/2,-\heightQ/2) ;
    \draw (\deltaFig+2.5+\distLayer + \widthLayer/2+ \distLayer,\rayonSign) -- (\deltaFig+2.5+3*\distLayer +\widthLayer/2,\heightQ/2) ;
    
    \fill [colorOutput,rounded corners=3.1, draw=black, line width=\linewidthLayer]
        (\deltaFig+2.5+3*\distLayer,-\heightQ/2) --
        ++(\widthLayer,0) --
        ++(0,\heightQ) --
        ++(-\widthLayer,0) --
        cycle
        {};
\end{tikzpicture}
\caption{Dense and Multiplicative Dueling architectures for the Q-network. The rectangles represent consecutive dense layers, the lightblue block being fed with all the features whereas the lightgreen one is only fed with static features (darkgreen). The output of the MDA is the product between a single unit and a $\abs{\JJ}$-unit dense layer.} 
\label{fig:MDA}
\end{figure}
\vspace{-.7cm}

\section{Experiments and discussion}

We test our algorithms on two sets of instances provided by Electricit\'e de France (EDF), a french electric utility company. They are drawn from two different problems, one is related to energy management in a microgrid ($\mathcal{P}_1$) whereas the other one comes from a hydroelectric valley ($\mathcal{P}_2$). The problems have respectively 186 and 282 constraints, 120 and 207 variables, and 54 and 96 binary variables.\\
We compare our algorithms to the default branching strategy of CPLEX (denoted CPLEX in the following) and full Strong Branching (denoted SB). We use CPLEX 12.7.1~\cite{manual_cplex} under DFS while turning off all presolving and cutting.\\
To avoid any dependency of our results to the train or the test set, we present cross-validated results. Algorithm~\ref{algo:FMSTS} is run 100 times independently on randomly partitioned train and test sets. Each time, 200 instances are used for training while 500 unseen instances are used for testing. Figure~\ref{fig:results} shows the averaged number of nodes in the complete B\&B trees on the test sets during the learning process: test instances are solved using the strategy learned on train instances at the current iteration of Algorithm~\ref{algo:FMSTS}.\\

\begin{figure}
    \centering
    \begin{tabular}{cc}
	\includegraphics[width=.5\textwidth]{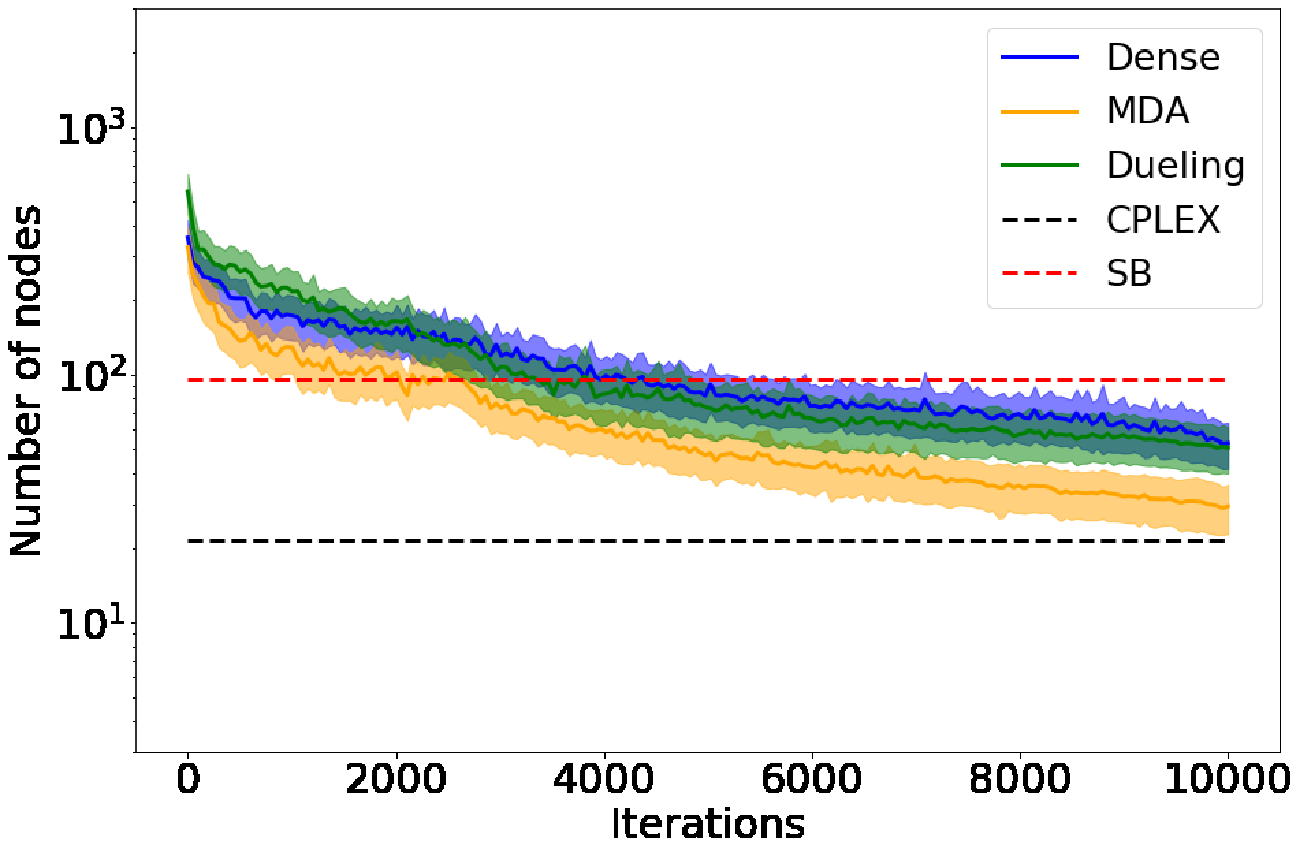}
 &  
     \includegraphics[width=.5\textwidth]{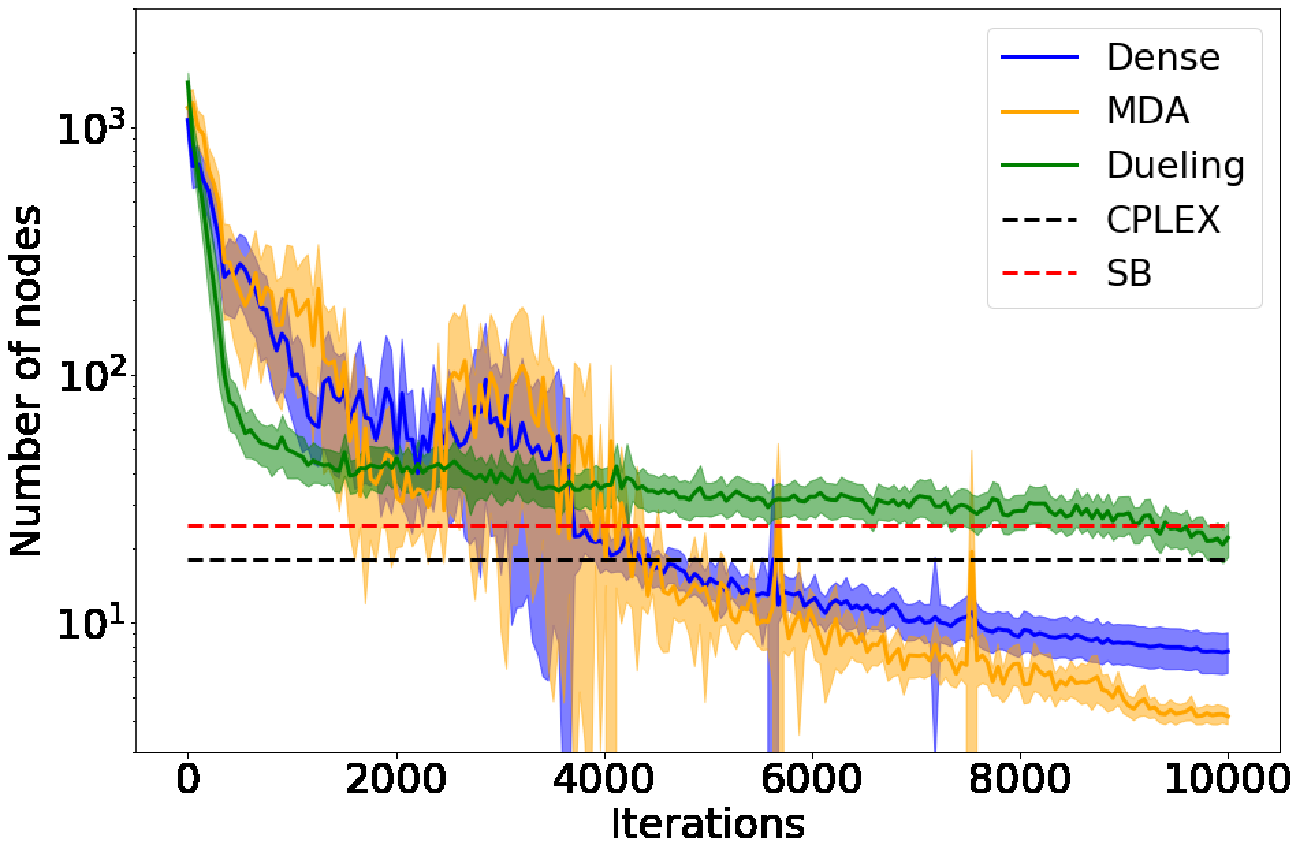}
\end{tabular}
\caption{Cross-validated performance on test instances (averaged number of nodes in log scale) for $\mathcal{P}_1$ (left) and $\mathcal{P}_2$ (right) through iterations of Algorithm~\ref{algo:FMSTS}. Gaussian confidence intervals are shown around the means.}
\label{fig:results}
\end{figure} 

As exhibited in Figure~\ref{fig:results}, our method is able to learn an efficient strategy from scratch. As expected, the MDA agent is more flexible than its additive counterpart (Dueling) inspired from the Dueling architecture of~\cite{dueling} and the fully-connected agent (Dense). It outperforms systematically the Strong Branching policy, and finds comparable or better strategies than CPLEX, depending on the problem. Results on training data are not displayed for the sake of conciseness, but it is worth mentioning that our agents do not overfit and are able to generalise well. In addition, the computation time is negligible compared with full Strong Branching as an action comes only at the price of a forward pass in our neural network. \\
Despite these good performances, some limits have to be pointed out at this stage. First, our framework requires DFS as the node selection strategy, which can be far from optimal for certain problems. Note that using another strategy may be complicated to handle due to more complex dependencies, but may also turn out to be effective as targetting small subtrees makes sense in general. Second, we only showed promising results on easy problems. With more difficult problems, the training becomes computationally prohibitive as a randomly initialised agent produces exponential trees. To tackle these limitations, we encompass different solutions such as fine-tuning the features and network architecture or using supervision to decrease the size of the generated trees during the first episodes. To reduce the cost of exploration, one could apply the same methodology with a set of branching heuristics as action set, similarly to what is proposed in~\cite{art:DASH}.

\section{Conclusion}

In this paper, we presented a novel Reinforcement Learning framework designed to learn from scratch the branching strategy in a B\&B algorithm. In addition to the specific metrics used in our FMSTS method, we introduced a new neural network architecture designed to tackle the multiplicative nature of the value function. Besides, we adapted some known RL techniques to the B\&B setting. We ran experiments on real-world problems to validate our method and showed better or comparable performances with existing strategies. \\
It is worthwhile to highlight that our method is generic enough to be applied to other metrics than the tree size, e.g. the number of simplex iterations or even the computation time. If one is not interested in proving optimality, many other value functions may be encompassed. Furthermore, it may be interesting to enlarge the scope of the method, especially to include Branch and Cut algorithms as they usually are more efficient.

\newpage
\bibliographystyle{ieeetr} 
\bibliography{bib}

\end{document}